\documentclass[10pt,twocolumn,letterpaper]{article}

\usepackage{cvpr}
\usepackage{times}
\usepackage{epsfig}
\usepackage{graphicx}
\usepackage{amsmath}
\usepackage{amssymb}
\usepackage{amsthm}
\usepackage{algorithm}
\usepackage{algorithmic}
\usepackage{multirow}
\usepackage{subfigure}
\usepackage{epstopdf}
\usepackage{enumitem}
\usepackage[page]{appendix}

\usepackage{tikz}
\usepackage{tikz-3dplot}
\usetikzlibrary{arrows}
\usepackage[pagebackref=true,breaklinks=true,letterpaper=true,colorlinks,bookmarks=false]{hyperref}

\cvprfinalcopy 

\def\cvprPaperID{5} 

\newtheorem{definition}{Definition}
\newtheorem{lemma}{Lemma}

\newtheorem{theorem}{Theorem}
\newtheorem{corollary}{Corollary}
\newtheorem*{lemma*}{Lemma}
\newtheorem*{theorem*}{Theorem}
\newtheorem*{corollary*}{Corollary}

\def\a{\boldsymbol{a}}
\def\b{\boldsymbol{b}}
\def\c{\boldsymbol{c}}
\def\q{\boldsymbol{q}}
\def\x{\boldsymbol{x}}
\def\y{\boldsymbol{y}}
\def\w{\boldsymbol{w}}
\def\cX{\mathcal{X}}
\def\cY{\mathcal{Y}}
\def\cV{\mathcal{V}}
\def\cW{\mathcal{W}}
\def\cP{\mathcal{P}}
\def\Re{\mathbb{R}}
\def\0{\mathbf{0}}
\def\OMP{\text{OMP}}
\DeclareMathOperator*{\card}{card}
\DeclareMathOperator*{\diag}{diag}
\DeclareMathOperator*{\rank}{rank}
\DeclareMathOperator*{\conv}{conv}
\DeclareMathOperator*{\argmin}{arg\,min}
\DeclareMathOperator*{\argmax}{arg\,max}
\def\transpose{\top} 
\newcommand{\myparagraph}[1]{\smallskip\noindent\textbf{#1.}}
\newcommand{\mysubparagraph}[1]{\smallskip\noindent-- \emph{#1:}}

\def\st{~~\mathrm{s.t.}~~}

\graphicspath{{figures/}}

\ifcvprfinal\fi
\begin{document}
	
	\title{Orthogonal Matching Pursuit for Sparse Subspace Clustering}
	\title{Scalable Sparse Subspace Clustering by Orthogonal Matching Pursuit}
	
	\author{Chong You, Daniel P. Robinson, Ren\'e Vidal\\
		Johns Hopkins University, Baltimore, MD, 21218, USA\\
	}
	
	\maketitle
	\thispagestyle{empty}

	\begin{abstract}
		Subspace clustering methods based on $\ell_1$, $\ell_2$ or nuclear norm regularization have become very popular due to their simplicity, theoretical guarantees and empirical success. However, the choice of the regularizer can greatly impact both theory and practice. For instance, $\ell_1$ regularization is guaranteed to give a subspace-preserving affinity (i.e., there are no connections between points from different subspaces) under broad conditions (\eg, arbitrary subspaces and corrupted data). However, it requires solving a large scale convex optimization problem. On the other hand, $\ell_2$ and nuclear norm regularization provide efficient closed form solutions, but require very strong assumptions to guarantee a subspace-preserving affinity, \eg, independent subspaces and uncorrupted data. In this paper we study a subspace clustering method based on orthogonal matching pursuit. We show that the method is both computationally efficient and guaranteed to give a subspace-preserving affinity under broad conditions. Experiments on synthetic data verify our theoretical analysis, and applications in handwritten digit and face clustering show that our approach achieves the best trade off between accuracy and efficiency.\!\!\!
	\end{abstract}

	\section{Introduction}
	In many computer vision applications, such as 
	motion segmentation \cite{Costeira:IJCV98,Vidal:IJCV08,Rao:PAMI10},
	hand written digit clustering \cite{Zhang:IJCV12} and face clustering
	\cite{Basri:PAMI03,Kriegman:PAMI05}, data from different classes can be well approximated by a union of low dimensional subspaces. In these scenarios, the task is to partition the data according to the membership of data points to subspaces. 
	
	More formally, given a set of points $\cX = \{\x_j \!\in\! \Re^D\}_{j=1}^N$ lying in an unknown number $n$ of subspaces $\{S_i\}_{i=1}^n$ of unknown dimensions $\{d_i\}_{i=1}^n$, \emph{subspace clustering} is the problem of clustering the data into groups such that each group contains only data points from the same subspace. This problem has received great attention in the past decade and many subspace clustering algorithms have been developed, including iterative, algebraic, statistical, and spectral clustering based methods (see \cite{Vidal:SPM11-SC} for a review).

	\myparagraph{Sparse and Low Rank Methods}
	Among existing techniques, methods based
	on applying spectral clustering to an affinity matrix obtained by
	solving an optimization problem that incorporates $\ell_1$, $\ell_2$ or nuclear norm regularization have become extremely popular due to their simplicity, theoretical correctness, and empirical success. These methods are based on the so-called \emph{self-expressiveness property} of data lying in a union of subspaces, originally proposed in \cite{Elhamifar:CVPR09}. This property states that each point in a union of subspaces can be written as a linear combination of other data points in the subspaces. That is, 
	\begin{equation}\label{eqn:self-expression}
	\begin{aligned}
	\x_j &= X \c_j \ \ \text{and} \ \  c_{jj} = 0, \ \ \text{or equivalently} \\
	X &= XC \ \ \text{and} \ \ \diag(C) = \0,
	\end{aligned}
	\end{equation}
	%
	where $X = \begin{bmatrix} \x_1, \dots, \x_N \end{bmatrix} \in \Re^{D \times N}$ is the data matrix and $C = \begin{bmatrix} \c_1, \dots, \c_N \end{bmatrix}\in\Re^{N\times N}$ is the matrix of coefficients. 
	
	While \eqref{eqn:self-expression} may not have a unique solution for $C$, there exist solutions whose entries are such that if $c_{ij}\neq 0$, then $\x_i$ is in the same subspace as $\x_j$. For example, a point $\x_j \in S_i$ can always be written as a linear combination of $d_i$ other points in $S_i$. Such solutions are called \emph{subspace preserving} since they preserve the clustering of the subspaces. Given a subspace preserving $C$, one can build an affinity matrix $W$ between every pair of points $\x_i$ and $\x_j$ as $w_{ij} = |c_{ij}| + |c_{ji}|$, and apply spectral clustering \cite{vonLuxburg:StatComp2007} to $W$ to cluster the data.
	
	To find a subspace preserving $C$, existing methods regularize $C$ with a norm $\| \cdot \|$, and solve a problem of the form:
	\begin{equation}
	\label{eqn:self-expression-regularized}
	C^* = \argmin_{C} \|C\| \st X = X C, ~~ \diag(C) = \0. 
	\end{equation}
	For instance, the sparse subspace clustering (SSC) algorithm \cite{Elhamifar:CVPR09} uses the $\ell_1$ norm to encourage the sparsity of $C$. Prior work has shown that SSC gives a subspace-preserving solution if the subspaces are independent \cite{Elhamifar:CVPR09,Elhamifar:TPAMI13}, or if the data from different subspaces satisfy certain separation conditions and data from the same subspace are well spread out \cite{Elhamifar:ICASSP10,Elhamifar:TPAMI13,You:ICML15,Soltanolkotabi:AS13}. Similar results exist for cases where data is corrupted by noise \cite{Wang-Xu:ICML13,Soltanolkotabi:AS14} and outliers \cite{Soltanolkotabi:AS13}.  Other self-expressiveness based methods use different regularizations on the coefficient matrix $C$. Least squares regression (LSR) \cite{Lu:ECCV12} uses $\ell_2$ regularization on $C$. Low rank representation (LRR)  \cite{Liu:ICML10,Liu:TPAMI13} and low rank subspace clustering (LRSC) \cite{Favaro:CVPR11,Vidal:PRL14} use nuclear norm minimization to encourage $C$ to be low-rank. Based on these, \cite{Lu:ICCV13-TraceLasso,Panagakis:PRL14,Lai:ECCV14,You:CVPR16-EnSC} study regularizations that are a mixture of $\ell_1$ and $\ell_2$, and \cite{Wang:NIPS13-LRR+SSC,Zhuang:CVPR12} propose regularizations that are a blend of $\ell_1$ and the nuclear norm.
	
	The advantage of $\ell_2$ regularized LSR and nuclear norm regularized LRR and LRSC over sparsity regularized SSC is that the solution for $C$ can be computed in closed form from the SVD of the (noiseless) data matrix $X$, thus they are computationally more attractive. However, the resulting $C$ is subspace preserving only when subspaces are independent and the data is uncorrupted. Thus, there is a need for methods that both guarantee a subspace-preserving affinity under broad conditions and are computationally efficient.
	
	\myparagraph{Paper Contributions}
	In this work we study the self-expressiveness based subspace clustering method that uses orthogonal matching pursuit (OMP) to find a sparse representation in lieu of the $\ell_1$-based basis pursuit (BP) method. The method is termed SSC-OMP, for its kinship to the original SSC, which is referred to as SSC-BP in this paper. 
	
	The main contributions of this paper are to find theoretical conditions under which the affinity produced by SSC-OMP is subspace preserving and to demonstrate its efficiency for large scale problems. Specifically, we show that:
	
	\begin{enumerate}[topsep=0.3em,noitemsep]
		\item When the subspaces and the data are \textbf{deterministic}, SSC-OMP gives a subspace-preserving $C$ if the subspaces are independent, or else if the subspaces are sufficiently separated and the data is well distributed.
		\vspace{0.3em}
		\item When the subspaces and data are drawn uniformly at \textbf{random}, SSC-OMP gives a subspace-preserving $C$ if the dimensions of the subspaces are sufficiently small relative to the ambient dimension by a factor controlled by the sample density and the number of subspaces.
		\vspace{0.3em}
		\item SSC-OMP is orders of magnitude \textbf{faster} than the original SSC-BP, and can handle up to $100,\! 000$ data points.
	\end{enumerate}
	
	\myparagraph{Related work}
	It is worth noting that the idea of using OMP for SSC had already been considered in \cite{Dyer:JMLR13}. The core contribution of our work is to provide much weaker yet more succinct and interpretable conditions for the affinity to be subspace preserving in the case of arbitrary subspaces. In particular, our conditions are naturally related to those for SSC-BP, which reveal insights about the relationship between these two sparsity-based subspace clustering methods. Moreover, our experimental results provide a much more detailed evaluation of the behavior of SSC-OMP for large-scale problems. It is also worth noting that conditions under which OMP gives a subspace-preserving representation had also been studied in \cite{You:ICML15}. Our paper presents a much more comprehensive study of OMP for the subspace clustering problem, by providing results under deterministic independent, deterministic arbitrary and random subspace models. In particular, our result for deterministic arbitrary models is much stronger than the main result in \cite{You:ICML15}.


	\section{SSC by Orthogonal Matching Pursuit}
	\label{sec:background}
	
	
	The SSC algorithm approaches the subspace clustering problem by finding a sparse representation of each point in terms of other data points. Since each point in $S_i$ can be expressed in terms of at most $d_i\ll N$ other points in $S_i$, such a sparse representation always exists. In principle, we can find it by solving the following optimization problem:
	\begin{equation}
	\c_j^* = \argmin_{\c_j} \|\c_j\|_0 \st \x_j = X \c_j, c_{jj} = 0,
	\label{eq:L_0}
	\end{equation}
	where $\|\c\|_0$ counts the number of nonzero entries in $\c$. Since this problem is NP hard, the SSC method in \cite{Elhamifar:CVPR09} relaxes this problem and solves the following $\ell_1$ problem:
	\begin{equation}
	\c_j^* = \argmin_{\c_j} \|\c_j\|_1 \st \x_j = X \c_j, ~~ c_{jj} = 0. 
	\label{eq:L_1}
	\end{equation}
	Since this problem is called the basis pursuit (BP) problem, we refer to the SSC algorithm in \cite{Elhamifar:CVPR09} as SSC-BP.

	The optimization problems \eqref{eq:L_0} and \eqref{eq:L_1} have been studied extensively in the compressed sensing community, see, \eg, the tutorials \cite{Elad:SIAM09,Candes:SPM08}, and it is well known that, under certain conditions on the dictionary $X$,  their solutions are the same. However, results from compressed sensing do not apply to the subspace clustering problem because when the columns of $X$ lie in a union of subspaces the solution for $C$ need not be unique (see Section \ref{sec:main-results} for more details). This has motivated extensive research on the conditions under which the solutions of \eqref{eq:L_0} or \eqref{eq:L_1} are useful for subspace clustering.
	
	It is shown in \cite{Elhamifar:CVPR09,Elhamifar:ICASSP10,Elhamifar:TPAMI13} that when the subspaces are either independent or disjoint, and the data are noise free and well distributed, both \eqref{eq:L_0} and \eqref{eq:L_1} provide a sparse representation $\c_j$ that is \emph{subspace preserving}, as defined next.
	\begin{definition}[Subspace-preserving representation] 
		\label{def:subspace-preserving}
		A representation $\c\in\Re^N$ of a point $\x\in S_i$ in terms of the dictionary $X=\begin{bmatrix} \x_1 , \ldots, \x_N \end{bmatrix}$ is called subspace preserving if its nonzero entries correspond to points in $S_i$, \ie
		\begin{align}
		\forall j=1,\dots,N, \quad c_j \neq 0 \implies  \x_j \in S_i.
		\end{align}
		%
	\end{definition}
	

	In practice, however, solving $N$ $\ell_1$-minimization problems over $N$ variables may be prohibitive when $N$ is large. As an alternative, consider the following program:
	\begin{equation}
	\c_j^* = \argmin_{\c_j} \| \x_j - X \c_j\|_2^2 \st\|\c_j\|_0 \leq k, ~ c_{jj} = 0.
	\label{eq:OMP}
	\end{equation}
	It is shown in \cite{Tropp:TIT04,Davenport:TIT10} that, under certain conditions, this problem can be solved using the orthogonal matching pursuit (OMP) algorithm \cite{Pati:Asilomar93} (Algorithm~\ref{alg:OMP}). OMP solves the problem $\min_{\c} \|A\c-\b\|_2^2 \st \|\c\|_0 \leq k$ greedily by selecting one column of $A = \begin{bmatrix} \a_1, \dots, \a_M \end{bmatrix}$ at a time (the one that maximizes the absolute value of the dot product with the residual in line \ref{step:best-residual}) and computing the coefficients for the selected columns until $k$ columns are selected. For subspace clustering purposes, the vector $\c_j^*\in\Re^N$ (the $j$th column of $C^* \in \Re^{N\times N}$), is computed as \pagebreak 
	$\text{OMP}(X_{-j}, \x_j)\in \Re^{N-1}$ with a zero inserted in its $j$th entry, where $X_{-j}$ is the data matrix with the $j$th column removed. After $C^*$ is computed, the segmentation of the data is found by applying spectral clustering to the affinity matrix $W = |C^*|+|C^{*\top}|$ as done in SSC-BP. The procedure is summarized in Algorithm \ref{alg:SSC-OMP}.

	\renewcommand{\algorithmicrequire}{\textbf{Input:}}
	\renewcommand{\algorithmicensure}{\textbf{Output:}}
	\setlength{\textfloatsep}{8pt}
	\begin{algorithm}[t]
		\caption{\bf : Orthogonal Matching Pursuit (OMP)}
		\label{alg:OMP}
		\begin{algorithmic}[1]
			\REQUIRE $A = [\a_1, \dots, \a_M] \in \Re ^{m \times M}$, $\b \in \Re ^m$, $k_{\max}$, $\epsilon$.
			\STATE \label{step:initialize}\textbf{Initialize} $k = 0$, residual $\q_0 = \b$, support set $T_0 = \emptyset$.
			\WHILE {$k < k_{\max}$ and $\|\q _{k}\|_2 > \epsilon$}
			\STATE \label{step:best-residual}
			$T _{k+1} = T _{k} \bigcup \{i^*\}$, where
			$i^* = \argmax\limits_{i = 1, \ldots, M} |\a _i ^T \q _{k}|$\footnotemark[1]. 
			\STATE \label{step:residual}
			$\q _{k+1} = (I - P_{T _{k+1}}) \b$, where $P_{T _{k+1}}$ is the projection onto the span of the vectors $\{\a_j, j \in T_{k+1} \}$.\!\!\!
			\STATE $k \leftarrow k+1$. 
			\ENDWHILE
			\ENSURE \label{step:output} $\c^* = \argmin _{ \c: \text{Supp}(\c) \subseteq T _{k} } \|\b - A\c\| _2$.
		\end{algorithmic}
	\end{algorithm}

	\footnotetext[1]{If $\argmax$ in step \ref{step:best-residual} of the algorithm gives multiple items, pick one of them in a deterministic way, \eg, pick the one with the smallest index.}

	\begin{algorithm}[t]
		\caption{\bf : Sparse Subspace Clustering by Orthogonal Matching Pursuit (SSC-OMP)}
		\label{alg:SSC-OMP}
		\begin{algorithmic}[1]
			\REQUIRE Data $X = [\x_1, \cdots, \x_N]$, parameters $k_{\max}, \epsilon$.
			\STATE Compute $\c_j^*$ from $\text{OMP}(X_{-j},\x_j)$ using Algorithm \ref{alg:OMP}.
			\STATE Set $C^* = [\c_1^*, \cdots, \c_N^*]$ and $W = |C^*|+|C^{*\transpose}|$.
			\STATE Compute segmentation from $W$ by spectral clustering.
			\ENSURE Segmentation of data $X$.
		\end{algorithmic}
	\end{algorithm}

	\section{Theoretical Analysis of SSC-OMP}
	\label{sec:main-results}
	
	OMP has been shown to be effective for sparse recovery, with the advantage over BP that it admits simple, fast implementations. However, note that existing conditions for the correctness of OMP for sparse recovery are too strong for the subspace clustering problem. In particular, note that the matrix $X$ need not satisfy the mutual incoherence \cite{Tropp:TIT04} or restricted isometry properties \cite{Davenport:TIT10}, as two points in a subspace could be arbitrarily close to each other. More importantly, these conditions are not applicable here because our goal is not to recover a \emph{unique} sparse solution. In fact, the sparse solution is not unique since any $d_i$ linearly independent points from $S_i$ can represent a point $\x_j \in S_i$. Therefore, there is a need to find conditions under which the output of OMP (which need not coincide with the solution of \eqref{eq:OMP} or \eqref{eq:L_0}) is guaranteed to be subspace preserving.
	
	This section is devoted to studying sufficient conditions under which SSC-OMP gives a subspace-preserving representation. Our analysis assumes that the data is noiseless. The termination parameters of Algorithm \ref{alg:OMP} are $\epsilon = 0$ and $k_{max}$  large enough (\eg, $k_{max} = M$). We also assume that the columns of $X$ are normalized to unit $\ell_2$ norm. To make our results consistent with state-of-the-art results, we first study the case where the subspaces are deterministic, including both independent subspaces as well as arbitrary subspaces. We then study the case where both the subspaces and the data points are drawn at random. 
	
	\subsection{Independent Deterministic Subspace Model}
	
	We first consider the case where the subspaces are fixed, the data points are fixed, and the subspaces are independent.
	
	\begin{definition}
		A collection of subspaces $\{S_i\}_{i=1}^n$ is called independent if $\dim\big(\sum_i S_i \big ) = \sum_i \dim(S_i)$, where $\sum_i S_i$ is defined as the subspace $\{ \sum_i \x_i: \x_i \in S_i \}$. 
	\end{definition}
	
	Notice that two subspaces are \emph{independent} if and only if they
	are \emph{disjoint}, \ie, if they intersect only at the
	origin. However, pairwise disjoint subspaces need not be independent,
	\eg, three lines in $\Re^2$ are disjoint but not independent. Notice
	also that any subset of a set of independent subspaces is also
	independent. Therefore, any two subspaces in a set of independent
	subspaces are independent and hence disjoint. In particular, this
	implies that if $\{S_i\}_{i=1}^n$ are independent, then
	$S_i$ and $S_{(-i)} := \sum_{m \ne i} S_m$ are independent. 
	
	To establish conditions under which SSC-OMP gives a subspace-preserving affinity for independent subspaces, it is important to note that when computing $\OMP(X_{-j},\x_j)$, the goal is to select other points in the same subspace as $\x_j$. The process for selecting these points occurs in step \ref{step:best-residual} of Algorithm \ref{alg:OMP}, where the dot products between all points $\x_m$, $m\neq j$, and the current residual $\q_k$ are computed and the point with the highest product (in absolute value) is chosen. Since in the first iteration the residual is $\q_0 = \x_j$, we could immediately choose a point $\x_m$ in another subspace whenever the dot product of $\x_j$ with a point in another subspace is larger than the dot product of $\x_j$ with points in its own subspace. What the following theorem shows is that, even though OMP may select points in the wrong subspaces as the iterations proceed, the coefficients associated to points in other subspaces will be zero at the end. 
	Therefore, OMP (with $\epsilon=0$ and $k_{\max}=N-1$) is guaranteed to find a subspace-preserving representation.
	\begin{theorem}
		\label{thm:SSC-OMP-independent}
		If the subspaces are independent, OMP gives a subspace-preserving representation of each data point.
	\end{theorem}
	\begin{proof}{[Sketch only]}
		Assume that $\x_j \in S_i$. Since $\epsilon=0$ and $k_{\max}$ is large, OMP gives an exact representation, \ie, $\x_j = X \c_j$ and $c_{jj}=0$. Thus, since $S_i$ and $S_{(-i)}$ are independent, the coefficients of data points in $S_{(-i)}$ must~be~zero.
	\end{proof}

	\subsection{Arbitrary Deterministic Subspace Model}
	
	We will now consider a more general class of subspaces, which need not be independent or disjoint, and investigate conditions under which OMP gives a subspace-preserving representation. In the following, $X^{i} \in \Re^{D\times N_i}$ denotes the submatrix of $X$ containing the points in the $i$th subspace; for any $\x_j \in
	S_i$, $X_{-j}^i\in\Re^{D\times (N_i-1)}$ denotes the matrix $X^{i}$ with the point $\x_j$ removed; $\cX^{i}$ and $\cX_{-j}^i$ denote respectively the set of vectors contained in the columns of $X^{i}$ and $X_{-j}^i$.
	
	Now, it is easy to see that a sufficient condition for $\OMP(X_{-j}, \x_j)$ to be subspace preserving is that for each $k$ in step \ref{step:best-residual} of Algorithm \ref{alg:OMP}, the point that maximizes the dot product lies in the same subspace as $\x_j$. Since $\q_0 = \x_j$ and $\q_1$ is equal to $\x_j$ minus the projection of $\x_j$ onto the subspace spanned by the selected point, say $\hat{\x}$, it follows that if $\x_j,\hat{\x} \in S_i$ then $\q_1\in S_i$. By a simple induction argument, it follows that if all the selected points are in $S_i$, then so are the residuals $\{\q_k\}$. This suggests that the condition for $\OMP(X_{-j}, \x_j)$ to be subspace preserving must depend on the dot products between the data points and a subset of the set of residuals (the subset contained in the same subspace as $\x_j$). This motivates the following definition and lemma.

	\begin{definition}
		\label{def:residual-direction}
		Let $Q(A, \b)$ be the set of all residual vectors computed in step \ref{step:residual} of $\OMP(A, \b)$. The set of OMP residual directions associated with matrix $X _{-j}^{i}$ and point $\x_j \!\in\! S_i$ is defined as:
		\begin{equation}
		\cW_j^i := \{\w = \frac{\q}{\|\q\|_2}: \q \in Q(X _{-j}^i, \x_j), \q \ne \0\}.
		\end{equation}
		The set of OMP residual directions associated with the data matrix $X^i$ is defined as $\cW^i := \bigcup_{j: \x_j \in S_i} \cW_j^i$.
	\end{definition}
	

	\begin{lemma}
		\label{lem:ss_individual}
		OMP gives a subspace-preserving representation for point $\x_j \in S_i$ in at most $d_i$ iterations if 
		\begin{equation}
		\label{eqn:ss_individual}
		\forall \w \in \cW_j^i
		\quad
		\max_{\x \in \bigcup_{k \ne i} \cX^k} | \w^\transpose \x | < \max_{\x \in \cX^i \backslash \{\x_j\}}|\w^\transpose \x | .
		\end{equation}
	\end{lemma}
	\begin{proof}{[Sketch only]}
		By using an induction argument, it is easy to see that the condition in \eqref{eqn:ss_individual} implies that the sequence of residuals of $\OMP(X_{-j}, \x_j)$ is the same as that of the fictitious problem $\OMP(X_{-j}^i, \x_j)$. Hence, the output of $\OMP(X_{-j}, \x_j)$ is the same as that of $\OMP(X_{-j}^i, \x_j)$, which is, by construction, subspace-preserving.
	\end{proof}

	Intuitively, Lemma \ref{lem:ss_individual} tells us that if the dot product between the residual directions for subspace $i$ and the data points in all other subspaces is smaller than the dot product between the residual directions for subspace $i$ and all points in subspace $i$ other than $\x_j \in S_i$, then OMP gives a subspace-preserving representation. While such a condition is very intuitive from the perspective of OMP, it is not as intuitive from the perspective of subspace clustering as it does not rely on the geometry of the problem. Specifically, it does not directly depend on the relative configuration of the subspaces or the distribution of the data in the subspaces. In what follows, we derive conditions on the subspaces and the data that guarantee that the condition in \eqref{eqn:ss_individual} holds. Before doing so, we need some additional definitions.
	
	\begin{definition}
		\label{def:coherence}
		The coherence between two sets of points of unit norm, $\cX$ and $\cY$, is defined as
		$
		\mu(\cX, \cY) = \max_{\x \in \cX, \y \in \cY} |\langle \x, \y \rangle|.
		$
	\end{definition}
	
	The coherence measures the degree of ``similarity" between two sets of
	points. In our case, we can see that the left hand side of
	\eqref{eqn:ss_individual} is bounded above by the coherence between
	the sets $\cW^i$ and $\bigcup_{k \ne i}\cX^{k}$. As per
	\eqref{eqn:ss_individual}, this coherence should be small, which
	implies that data points from different subspaces should be
	sufficiently separated (in angle).
	
	\begin{definition}
		The inradius $r(\cP)$ of a convex body $\cP$ is the radius of the largest Euclidean ball inscribed in $\cP$.
	\end{definition}
	
	As shown in Lemma \ref{lem:chain-inequalities}, the right hand side of
	\eqref{eqn:ss_individual} is bounded below by  $r(\cP_{-j}^i)$, where
	$\cP_{-j}^i := \conv\big(\pm \cX_{-j}^i\big)$ is the symmetrized
	convex hull of the points in the $i$th subspace other than $\x_j$,
	\ie, $\cX_{-j}^i$. Therefore, \eqref{eqn:ss_individual} suggests that
	the minimum inradius $r_i := \min_{j} r(\cP_{-j}^i)$ should be large,
	which means the points in $S_i$ should be well-distributed.
	
	
	\begin{lemma}
		\label{lem:chain-inequalities}
		Let $\x_j \in S_i$. Then, for all $\w \in \cW_j^i$, we have:
		\begin{align}
		&\max_{\x \in \bigcup_{k \ne i} \!\cX^k} \!\! |\w^\transpose \x|
		\le \!\max_{k: k \ne i} \mu(\cW^i, \cX^k)
		\le \!\max_{k: k \ne i} \mu(\cX^i, \cX^k) / r_i;\nonumber\\
		&\max_{\x \in \cX^i \backslash \{\x_j\}} |\w ^\transpose \x|
		\ge r(\cP_{-j}^i) \ge r_i.
		\end{align}
	\end{lemma}
	\begin{proof}
		The proof can be found in the Appendix. 
	\end{proof}
	
	Lemma \ref{lem:chain-inequalities} allows us to make the condition of Lemma \ref{lem:ss_individual} more interpretable, as stated in the following theorem.
	
	\begin{theorem}
		\label{thm:deterministic_1}
		The output of OMP is subspace preserving if
		%
		\begin{equation}
		\forall i=1,\dots,n,\quad
		\max_{k: k \ne i} \mu(\cW^i,\cX ^k) < r_i.
		\label{eq:result_1}
		\end{equation} 
	\end{theorem}
	
	\begin{corollary}
		\label{thm:deterministic_2}
		The output of OMP is subspace preserving if
		\begin{equation}
		\forall i=1,\dots,n,\quad
		\max_{k: k \ne i} \mu(\cX^i, \cX ^k) < r_i ^2.
		\label{eq:result_2}
		\end{equation} 
	\end{corollary}
	
	Note that points in $\cW ^i$ are all in subspace $S_i$, as step \ref{step:residual} of $\OMP(A:=X_{-j} ^i, \b:=\x_j)$ has $\b$ and $P_{T_{k+1}} \b$ both in $S_i$. The conditions \eqref{eq:result_1} and \eqref{eq:result_2} thus show that for each subspace $S_i$, a set of points (\ie, $\cX^i$ or $\cW^i$) in $S_i$ should have low coherence with all points from other subspaces, and that points in $\cX^i$ should be uniformly located in $S_i$ to have a large inradius. This is in agreement with the intuition that points from different subspaces should be well separated, and points within a subspace should be well distributed.

	For a comparison of Corollary \ref{thm:deterministic_2} and Theorem \ref{thm:deterministic_1}, note that due to Lemma \ref{lem:chain-inequalities} condition  \eqref{eq:result_1} is tighter than condition \eqref{eq:result_2}, making Theorem \ref{thm:deterministic_1} preferable. Yet Corollary \ref{thm:deterministic_2} has the advantage that both sides of condition \eqref{eq:result_2} depend directly on the data points in $\cX$, while condition \eqref{eq:result_1} depends on the residual points in $\cW ^i$, making it algorithm specific.
	
	Another important thing to notice is that  conditions \eqref{eq:result_1} and \eqref{eq:result_2} can be satisfied even if the subspaces are \emph{neither independent nor disjoint}. For example, consider the case where $S_i\bigcap S_k \neq \0$. Then, the coherence $\mu(\cW^i, \cX ^k)$ could still be small as long as no points in $\cW^i$ and $\cX ^k$ are near the intersection of $S_i$ and $S_k$. Actually, even this is too strong of an assumption since the intersection is a subspace. Thus, $\x \in \cX ^k$, $\y \in \cW ^i$ could both be very close to the intersection yet have low coherence. The same argument also works for condition \eqref{eq:result_2}. Admittedly, under specific distributions of points, it is possible that there exists $\x \in \cX ^k$ and $\y \in \cW ^i$ that are arbitrarily close to each other when they are near the intersection. However, this worst case scenario is unlikely to happen if we consider a random model, as discussed next.
	
	\subsection{Arbitrary Random Subspace Model}
	\label{sec:main-results-random}
	This section considers the fully random union of subspaces model in \cite{Soltanolkotabi:AS13}, where the basis elements of each subspace are chosen uniformly at random from the unit sphere of the ambient space and the data points from each subspace are uniformly distributed on the unit sphere of that subspace. Theorem \ref{thm:random1} shows that the sufficient condition in \eqref{eq:result_1} holds true with high probability (\ie the probability goes to 1 as the density of points grows to infinity) given some conditions on the subspace dimension $d$, the ambient space dimension $D$, the number of subspaces $n$ and the number of data points per subspace.
	
	\begin{theorem}
		\label{thm:random1}
		Assume a random union of subspaces model where all subspaces are of equal dimension $d$ and the number of data points in each subspace is $\rho d + 1$, where $\rho > 1$ is the ``density", so that the total number data points in all subspaces is  $N(n, \rho, d) = n  (\rho d + 1)$. The output of OMP is subspace preserving with probability $p > 1 - \frac{2d}{N(n, \rho, d) } - N(n, \rho, d)  e^ {-\sqrt{\rho} d}$ if
		\begin{equation}
		d < \frac{c^2(\rho) \log \rho}{12} \frac{D}{\log N(n, \rho, d) },
		\label{eq:result_rnd_1}
		\end{equation}
		where $c(\rho) > 0$ is a constant that depends only on $\rho$. 
	\end{theorem}
	
	%
	%
	
	One interpretation of the condition in \eqref{eq:result_rnd_1} is that the dimension $d$ of the subspaces should be small relative to the ambient dimension $D$. It also shows that as the number of subspaces $n$ increases, the factor $\log N(n, \rho, d) $ also increases, making the condition more difficult to be satisfied. In terms of the density $\rho$, it is shown in \cite{Soltanolkotabi:AS13} that there exists a $\rho_0$ such that $c(\rho) = 1/\sqrt{8}$ when $\rho > \rho_0$. Then, it is easy to see that when $\rho > \rho_0$, the term that depends on $\rho$ is $\frac{\log \rho}{\log N(n, \rho, d) } = \frac{\log \rho}{\log (n (\rho d + 1))}$, which is a monotonically increasing function of $\rho$. This makes the condition easier to be satisfied as the density of points in the subspaces increases. Moreover, the probability of success is $1 - \frac{2d}{N(n, \rho, d) } - N(n, \rho, d)  e^ {-\sqrt{\rho} d}$, which is also an increasing function of $\rho$ when $\rho$ is greater than a threshold value. As a consequence, as the density of the points increases, the condition in Theorem \ref{thm:random1} becomes easier to satisfy and the probability of success also increases. 
	
	\section{Relationships with Other Methods}
	\label{sec:comparison}
	
	In this section we compare our results for SSC-OMP with those for other methods of the general form in \eqref{eqn:self-expression-regularized}. These methods include SSC-BP \cite{Elhamifar:ICASSP10, Elhamifar:TPAMI13,Soltanolkotabi:AS13}, which uses the $\ell_1$ norm as a regularizer, LRR \cite{Liu:TPAMI13} and LRSC \cite{Vidal:PRL14}, which use the nuclear norm, and LSR \cite{Lu:ECCV12} which uses the $\ell_2$ norm. We also compare our results to those of \cite{Dyer:JMLR13} for SSC-OMP. The comparison is in terms of whether the solutions given by these alternative algorithms are subspace-preserving. 
	
	\myparagraph{Independent Subspaces} Independence is a strong assumption on the union of subspaces. Under this assumption, a subspace has a trivial intersection with not only every other subspace, but also the union of all other subspaces. This case turns out to be especially easy for a large category of self-expressive subspace clustering methods \cite{Lu:ECCV12}, and SSC-BP, LRR, LRSC and LSR are all able to give subspace-preserving representations. Thus, in this easy case, the proposed method is as good as state-of-the-art methods.
	
	
	\myparagraph{Arbitrary Subspaces} 
	To the best of our knowledge, when the subspaces are not independent, there is no guarantee~of correctness for LRR, LRSC and LSR. For SSC-BP,~as shown in \cite{Soltanolkotabi:AS13}, the representation is subspace-preserving if
	\begin{equation}
	\forall i = 1, \dots, n, \quad
	\max_{k: k \ne i} \mu(\cV^i, \cX ^k) < r_i, 
	\label{eq:SSC-BP-condition}
	\end{equation}
	where $\cV^i$ is a set of $N_i$ \emph{dual directions} associated with $X^i$. When comparing \eqref{eq:SSC-BP-condition} with our result in condition \eqref{eq:result_1}, we can see that the right hand sides are the same.  However, the left hand sides are not directly comparable, as no general relationship is known between the sets $\cV^i$ and $\cW^i$. Nonetheless, notice that the number of points in these two sets are not the same since $\card(\cV^i) = N_i$ and $\card(\cW^i) = N_i  d_i $. Therefore, if we assume that the points in $\cV^i$ and $\cW^i$ are distributed uniformly at random on the unit sphere, then $\mu(\cW^i, \cX ^k)$ is expected to be larger than $\mu(\cV^i, \cX ^k)$, making the condition for SSC-OMP less likely to be satisfied than that for SSC-BP. Now, when comparing \eqref{eq:SSC-BP-condition} with our condition in \eqref{eq:result_2}, we see that the left hand sides are comparable under a random model where both $\cV ^i$ and $\cX ^i$ contain $N_i$ points. However, the right hand side is $r_i ^2$, which is less than or equal to $r_i$ since the data are normalized and $r_i \le 1$. This again makes the condition for SSC-OMP more difficult to hold than that for SSC-BP. However, this difference is expected to vanish for large scale problems, and SSC-OMP is computationally more efficient, as we will see in Section \ref{sec:experiments}.
	
	\myparagraph{Random Subspaces} For the random model, \cite{Soltanolkotabi:AS13} shows that SSC-BP gives a subspace-preserving representation with probability $p > 1 - \frac{2}{N(n, \rho, d) } - N(n, \rho, d)  e^{-\sqrt{\rho} d}$
	if
	\begin{equation}
	d < \frac{c^2(\rho) \log \rho}{12} \frac{D}{\log N(n, \rho, d) }.
	\end{equation}
	If we compare this result with that of Theorem \ref{thm:random1}, we can see that the condition under which both methods succeed with high probability is exactly the same. The difference between them is that SSC-BP has a higher probability of success than SSC-OMP when $d > 1$. However,  it is easy to see that the difference in probability goes to zero as the density $\rho$ goes to infinity. This means that the performance difference vanishes as the scale of the problem increases.
	
	\myparagraph{Other Results for SSC-OMP}
	Finally, we compare our results with those in \cite{Dyer:JMLR13} for
	SSC-OMP. Define the principal angle between two subspaces $S_i$ and $S_k$ as:
	\begin{equation}
	\theta_{i, k}^* = \min_{\substack{\x \in S_i\\ \|\x\|_2 = 1}} \min_{\substack{\y \in S_k\\ \|\y\|_2 = 1}} \arccos\langle \x, \y \rangle.
	\end{equation}
	It is shown in \cite{Dyer:JMLR13} that the output of SSC-OMP is subspace-preserving if for all $i = 1, \dots, n$,
	\begin{equation}
	\label{eq:PriorOMP}
	\max_{k:k \ne i} \mu(\cX ^i, \cX ^k) < r_i - \frac{2 \sqrt{1 - (r_i)^2}}{\sqrt[4]{12}} \max_{k: k \ne i} \cos \theta_{i, k}^*.
	\end{equation}
	The merit of this result is that it introduces the subspace angles in the condition, and satisfies the intuition that the algorithm is more likely to work if the subspaces are far apart from each other. However, the RHS of the condition shows an intricate relationship between the intra-class property $r_i$ and the inter-class property $\theta^*_{i, k}$, which greatly complicates the interpretation of the condition. More importantly, as is shown in the Appendix, the condition is more restrictive than \eqref{eq:result_1}, which makes Theorem \ref{thm:deterministic_1} a stronger result.
	
	\section{Experiments}
	\label{sec:experiments}
	In this section, we first verify our theoretical results for SSC-OMP and compare them with those for SSC-BP by doing experiments on synthetic data using the random model. Specifically, we show that even if the subspaces are not independent,  the solution of OMP is subspace-preserving with a probability that grows with the density of data points. Second, we test the performance of the proposed method on clustering images of handwritten digits and human faces, and conclude that SSC-OMP achieves the best trade off between accuracy and efficiency.
	
	\myparagraph{Methods} We compare the performance of  state-of-the-art spectral subspace clustering methods, including LRSC \cite{Vidal:PRL14}, SSC-BP \cite{Elhamifar:TPAMI13}, LSR \cite{Lu:ECCV12}, and spectral curvature clustering (SCC) \cite{Chen:IJCV09}. In real experiments, we use the code provided by the respective authors for computing the representation matrix $C^*$, where the parameters are tuned to give the best clustering accuracy. We then apply the normalized spectral clustering in \cite{vonLuxburg:StatComp2007} to the affinity $|C^*|+|C^{*\transpose}|$, except for SCC which has its own spectral clustering step.

	\myparagraph{Metrics}
	We use two metrics to evaluate the degree to which the subspace-preserving property is satisfied. The first one is a direct measure of whether the solution is subspace preserving or not. However, for comparing with state of the art methods whose output is generally not subspace preserving, the second one measures how close the coefficients are from being subspace preserving.
	
	\mysubparagraph{Percentage of subspace-preserving representations ($p\%$)} this is the percentage of points whose representations are subspace-preserving. Due to inexactness in the solvers, coefficients with absolute value less than $10^{-3}$ are considered zero. A subspace-preserving solution gives $p = 100$.
	
	\mysubparagraph{Subspace-preserving representation error ($e\%$) \cite{Elhamifar:TPAMI13}} for each $\c_j$ in \eqref{eqn:self-expression}, we compute the fraction of its $\ell_1$ norm that comes from other subspaces and then average over all $j$, \ie, $e = \frac{100}{N}\sum_{j} (1 - \sum_i(\omega_{ij} \cdot |\c_{ij}|) / \|\c_j\|_1)$, where $\omega_{ij} \in \{0,1\}$ is the true affinity. A subspace-preserving $C$ gives $e = 0$.
	
	Now, the performance of subspace clustering depends not$\!$ only on the subspace-preserving property, but also the connectivity of the similarity graph, \ie, whether the data points in each cluster form a connected component of the graph.
	
	\mysubparagraph{Connectivity ($c$)} 
	For an undirected graph with weights $W \in \Re ^{N \times N}$ and
	degree matrix $D = \text{diag} (W \cdot \mathbf{1})$, where
	$\mathbf{1}$ is the vector of all ones, we use the second smallest
	eigenvalue $\lambda_2$ of the normalized Laplacian $L = I - D^{-1/2} W
	D^{-1/2}$ to measure the connectivity of the graph; $\lambda_2$ is in the range $[0, \frac{n-1}{n}]$ and is zero if and only if the graph is not connected \cite{Fiedler:CMathJ1975,Chung:1997}.
	In our case, we compute the algebraic connectivity for each cluster, $\lambda_2^i$, and take the quantity $c = \min_i \lambda_2^i$ as the measure of connectivity.
	
	Finally, we use the following two metrics to evaluate the performance of subspace clustering methods.
	
	\mysubparagraph{Clustering accuracy ($a\%$)} this is the percentage of
	correctly labeled data points. It is computed by matching the
	estimated and true labels as $a  = \max\limits_{\pi} {100\over N}
	\sum_{ij} Q_{\pi(i)j}^{est} Q_{ij}^{true}$, where $\pi$ is a
	permutation of the $n$ groups, $Q^{est}$ and $Q^{true}$ are the estimated and ground-truth labeling of data, respectively, with their $(i,j)$th entry being equal to one if point $j$ belongs to cluster $i$ and zero otherwise.
	
	\mysubparagraph{Running time ($t$)} for each clustering task using \textregistered Matlab.
	
	The reported numbers in all the experiments of this section are averages over 20 trials.

	\subsection{Synthetic Experiments}
	
	We randomly generate $n=5$ subspaces each of dimension $d = 6$ in an
	ambient space of dimension $D = 9$. Each subspace contains $N_i = \rho
	d$ sample points randomly generated on the unit sphere, where $\rho$
	is varied from $5$ to $3,\! 333$, so that the number of points varies
	from $150$ to $99,\! 990$. For SSC-OMP, we set $\epsilon$ in Algorithm \ref{alg:OMP} to be $10^{-3}$ and $k_{\max}$ to be $d=6$. For SSC-BP we use the $\ell_1$-Magic solver. Due to the computational complexity, SSC-BP
	is run for $\rho\leq 200$.  
	
	\begin{figure*}
		\centering
		\subfigure[\label{fig:synthetic-result-p} Subspace-preserving representation percentage]{\quad\includegraphics[scale = 0.35]{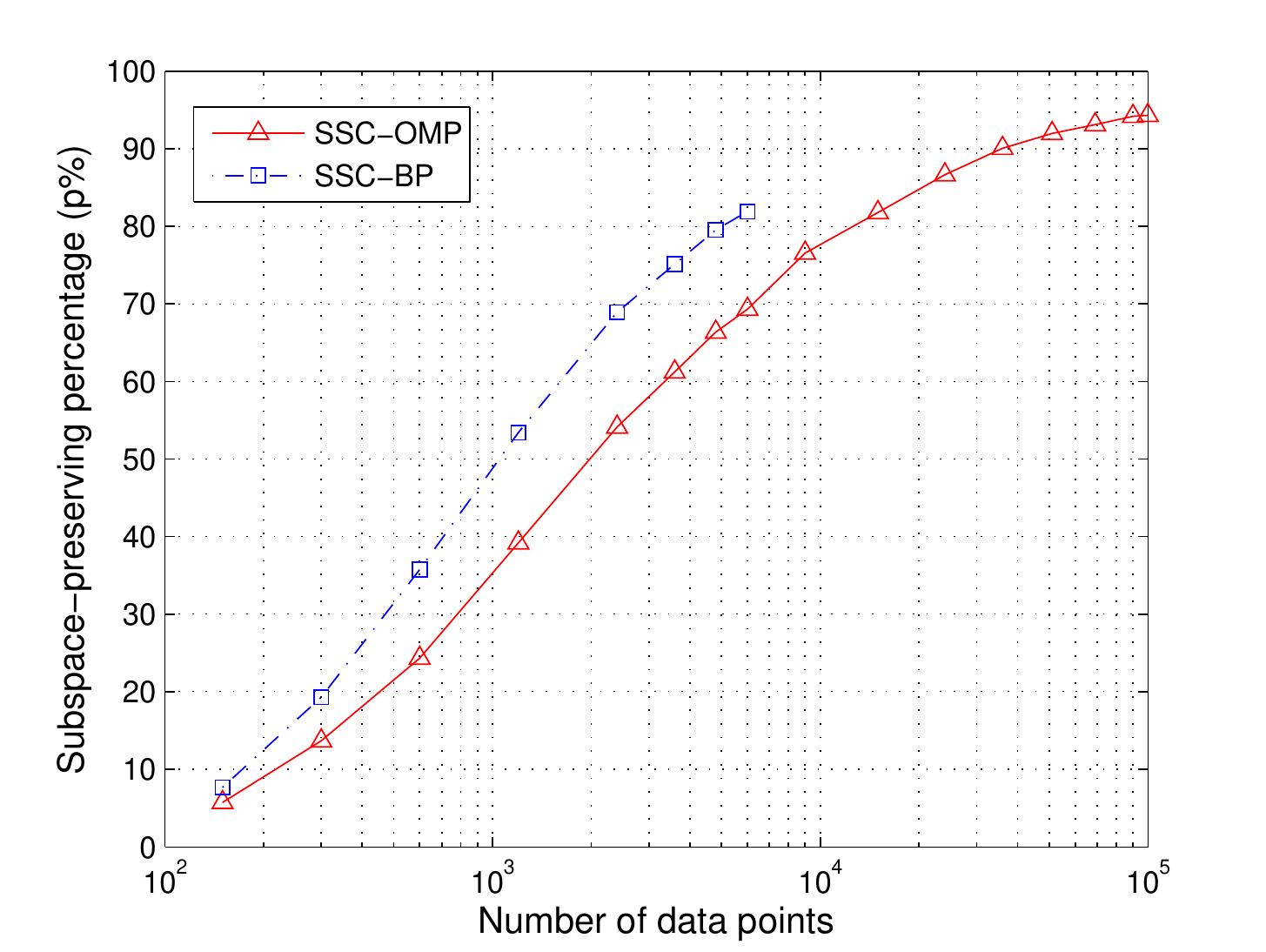}\quad}
		~
		\subfigure[\label{fig:synthetic-result-e}
		Subspace-preserving representation error]{\includegraphics[scale = 0.35]{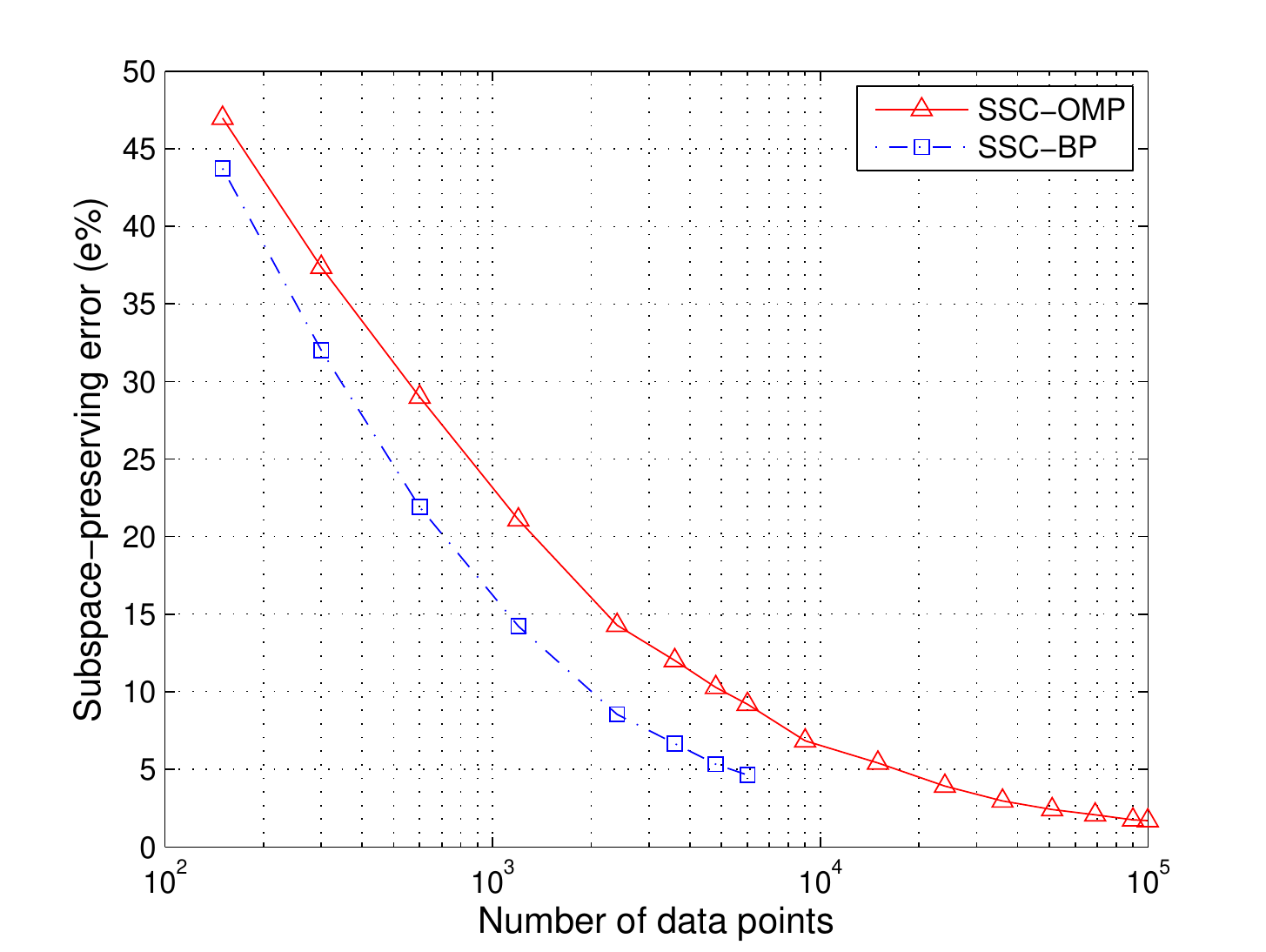}}
		\\
		\subfigure[\label{fig:synthetic-result-c}
		Connectivity]{\includegraphics[scale = 0.35]{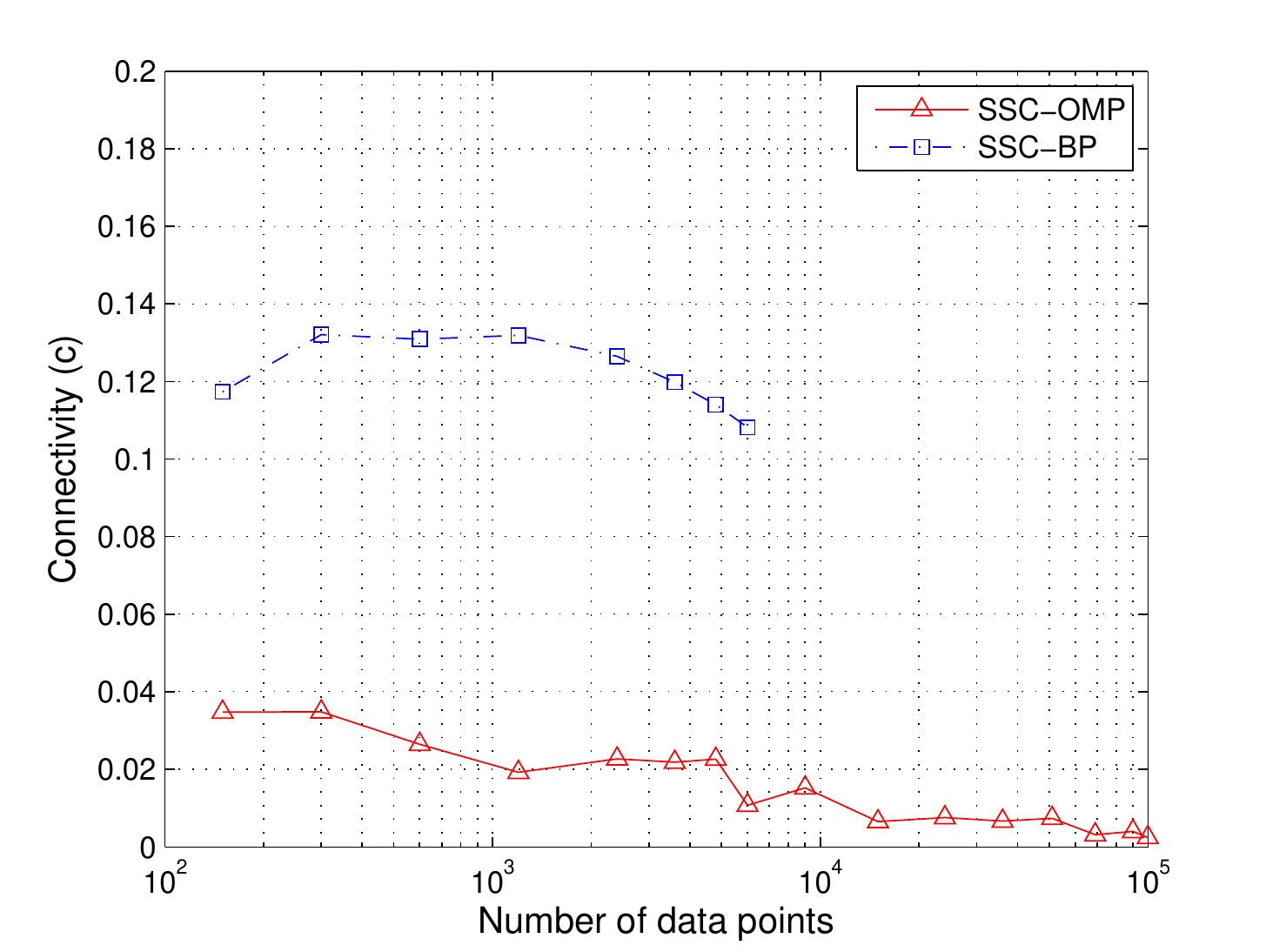}}
		~
		\subfigure[\label{fig:synthetic-result-a}
		Clustering accuracy]{\includegraphics[scale = 0.35]{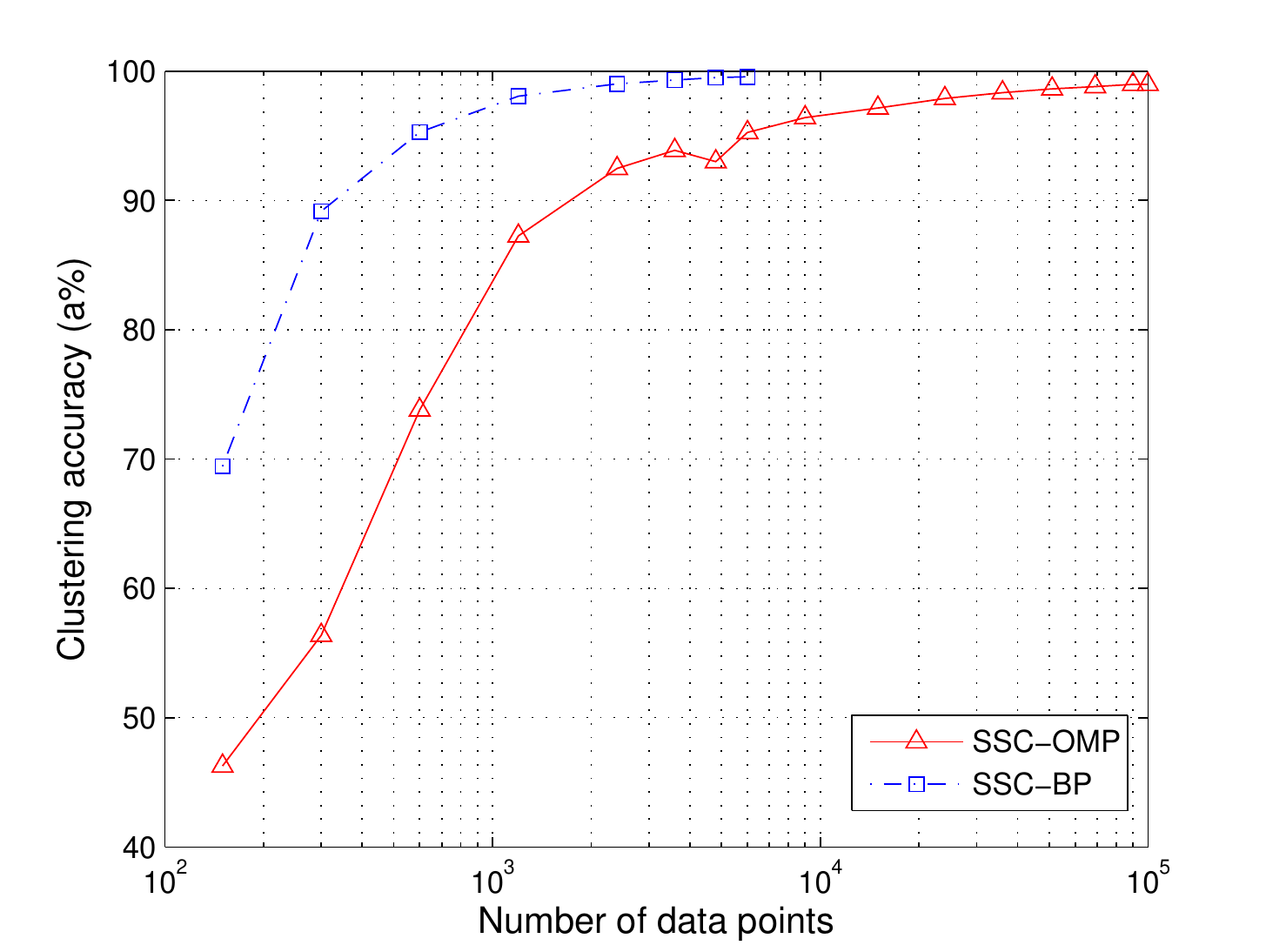}}
		~
		\subfigure[\label{fig:synthetic-result-t}Computational time]{\includegraphics[scale = 0.35]{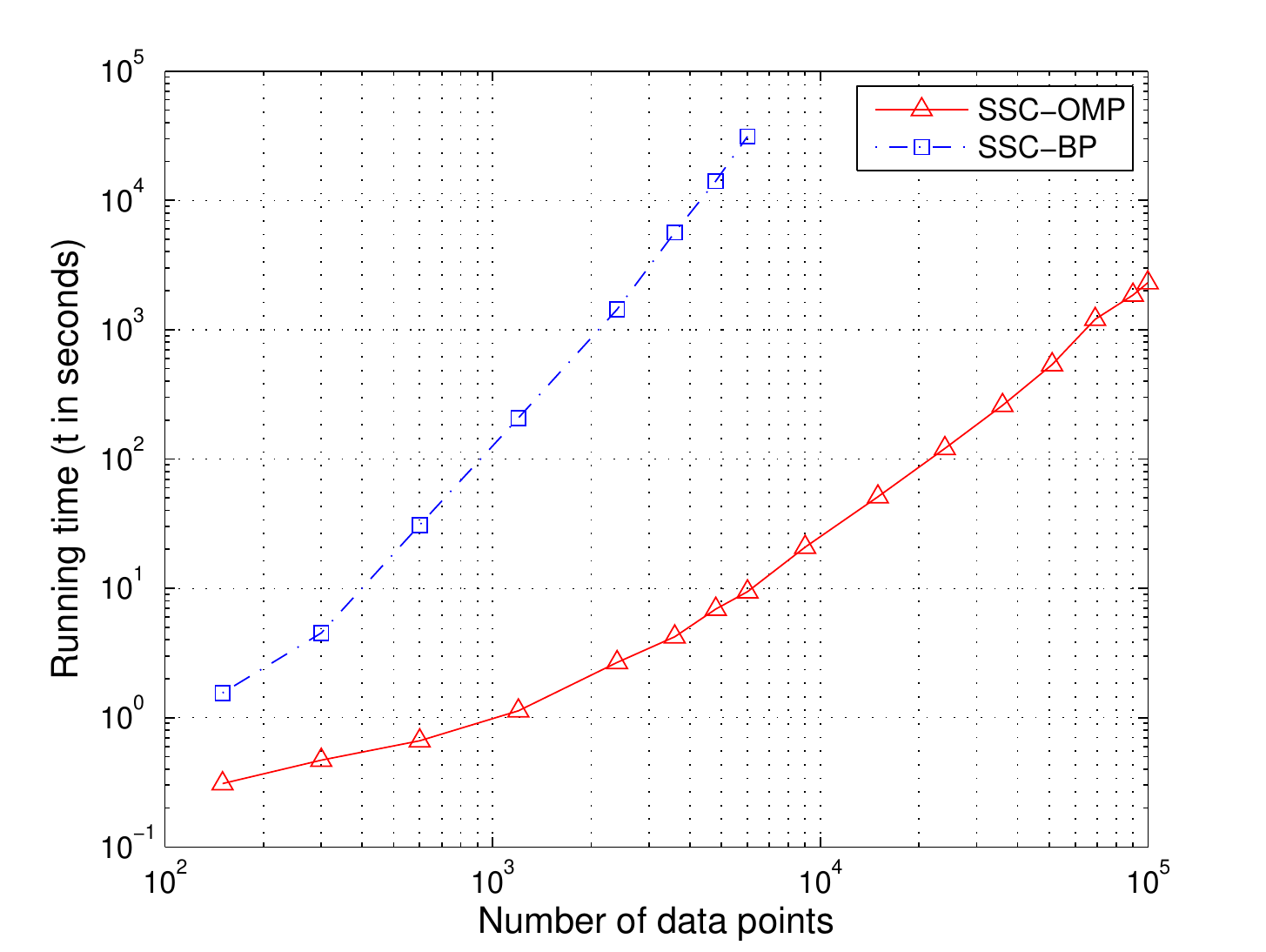}}
		\caption{Performance of SSC-OMP and SSC-BP on synthetic
			data. The data are drawn from 5 subspaces of dimension 6 in
			ambient dimension 9. Each subspace contains the same number
			of points and the overall number of points is varied from
			$150$ to $10^5$ and is shown in log scale. For SSC-BP,
			however, the maximum number of points tested is $6,\!000$
			due to time limit. Notice that the bottom right figure also
			uses log scale in the y-axis.}
		\label{fig:synthetic-result}
		\vspace{-1mm}
	\end{figure*}
	
	The subspace-preserving representation percentage and error are plotted in Figure \ref{fig:synthetic-result-p} and \ref{fig:synthetic-result-e}. Observe that the probability that SSC-OMP gives a subspace-preserving solution grows as the density of data point increases. When comparing with SSC-BP, we can see that SSC-OMP is outperformed. This matches our analysis that the condition for SSC-OMP to give a subspace-preserving representation is stronger (\ie, is more difficult to be satisfied). 
	
	From a subspace clustering perspective, we are more interested in how well the method performs in terms of clustering accuracy, as well as how efficient the method is in terms of running time. These results are plotted in Figure \ref{fig:synthetic-result-a} and \ref{fig:synthetic-result-t}, together with the  connectivity \ref{fig:synthetic-result-c}. We first observe that SSC-OMP does not have as good a connectivity as SSC-BP. This could be partly due to the fact that it has fewer correct connections in the first place as shown by the subspace-preserving percentage. For clustering accuracy, SSC-OMP is also outperformed by SSC-BP. This comes at no surprise as the sparse representations produced by SSC-OMP are not as subspace-preserving or as well connected as those of SSC-BP. However, we observe that as the density of data points increases, the difference in clustering accuracy also decreases, and SSC-OMP seems to achieve arbitrarily good clustering accuracy for large $N$. Also, it is evident from Figure~\ref{fig:synthetic-result-t} that SSC-OMP is significantly faster: it is $3$ to $4$ orders of magnitude faster than SSC-BP when clustering $6,\!000$ points. We conclude that as $N$ increases, the difference in clustering accuracy between SSC-OMP and SSC-BP reduces, yet SSC-OMP is significantly faster, which makes it preferable for large-scale problems.
	
	\subsection{Clustering Images of Handwritten Digits}
	
	In this experiment, we evaluate the performance of different subspace clustering methods on clustering images of handwritten digits. We use the MNIST dataset \cite{LeCun:1998}, which contains grey scale images of handwritten digits $0-9$. 
	
	In each experiment, $N_i\in\{50,100,200,400,600\}$ randomly chosen images for each of the $10$ digits are chosen. For each image, we compute a set of feature vectors using a scattering convolution network \cite{Bruna:PAMI13}. The feature vector is a concatenation of coefficients in each layer of the network, and is translation invariant and deformation stable. Each feature vector is of size $3,\! 472$. The feature vectors for all images are then projected to dimension $500$ using PCA. The subspace clustering techniques are then applied to the projected features. The results are reported in Table \ref{tbl:MNIST-result}.
	
	\begin{table}[t]
		\centering
		\caption{Performance of subspace clustering methods on the MNIST dataset. The data consists of a randomly chosen number $N_i \in \{50, 100, 200, 400, 600\}$ of images for each of the $10$ digits (\ie, $0$-$9$), with features extracted from a scattering network and projected to dimension $500$ using PCA.}
		\label{tbl:MNIST-result}
		\begin{tabular}{c|ccccc}
			\hline
			No. points &  500  & 1000  & 2000  & 4000  &        6000         \\ \hline
			\multicolumn{6}{l}{\textsl{e\%: subspace-preserving representation error}} \\ \hline
			SSC-OMP   & 42.13 & 38.73 & 36.20 & 34.22 &        33.22        \\
			SSC-BP   & 29.56 & 24.88 & 21.07 & 17.80 &        16.08        \\
			LSR     & 78.24 & 79.68 & 80.83 & 81.75 &        82.18        \\
			LRSC    & 81.33 & 81.99 & 82.67 & 83.15 &        83.27        \\
			SCC     & 89.89 & 89.87 & 89.85 & 89.81 &        89.81        \\ \hline
			\multicolumn{6}{l}{\textsl{a\%: average clustering accuracy}}    \\ \hline
			SSC-OMP   & 83.64 & 86.67 & 90.60 & 91.22 &        91.25        \\
			SSC-BP   & 83.01 & 84.06 & 85.58 & 86.00 &        85.60        \\
			LSR     & 75.84 & 78.42 & 78.09 & 79.06 &        79.91        \\
			LRSC    & 75.02 & 79.76 & 79.44 & 78.46 &        79.88        \\
			SCC     & 53.45 & 61.47 & 66.43 & 71.46 &        70.60        \\ \hline
			\multicolumn{6}{l}{\textsl{t(sec.): running time}}               \\ \hline
			SSC-OMP   &  2.7  & 11.4  & 93.8  & 410.4 &        760.9        \\
			SSC-BP   & 20.1  & 97.9  & 635.2 & 4533  &        13605        \\
			LSR     &  1.7  &  5.9  & 42.4  & 136.1 &        327.6        \\
			LRSC    &  1.9  &  6.4  & 43.0  & 145.6 &        312.9        \\
			SCC     & 31.2  & 48.5  & 101.3 & 235.2 &        366.8        \\ \hline
		\end{tabular}
	\end{table}
	
	The numbers show that both SSC-OMP and SSC-BP give a much smaller subspace-preserving representation error than all other methods, with SSC-BP being better than SSC-OMP. This is consistent with our theoretical analysis as there is no guarantee that LSR or LRSC give a subspace-preserving representation for non-independent subspaces, and SSC-BP has a higher probability of giving a subspace-preserving representation than SSC-OMP.
	
	In terms of clustering accuracy, SSC-OMP is better than SSC-BP, which in turn outperforms LSR and LRSC, while SCC performs the worst among the algorithms tested. 
	
	Considering the running time of the methods, SSC-BP requires much more computation, especially when the number of points is large. Though SSC-OMP is an iterative method, its computation time is about twice that of LSR and LRSC, which have closed form solutions. This again qualifies the proposed method for large scale problems.

	\subsection{Clustering Face Images with Varying Lighting}
	
	In this experiment, we evaluate the performance of different subspace clustering methods on the Extended Yale B dataset \cite{Kriegman:PAMI01}, which contains frontal face images of $38$ individuals under $64$ different illumination conditions, each of size $192 \times 168$. 
	In this case, the data points are the original face images downsampled to $48 \times 42$ pixels. In each experiment, we randomly pick $n \in \{2, 10, 20, 30, 38\}$ individuals and take all the images (under different illuminations) of them as the data to be clustered.
	
	\begin{table}[t]
		\centering
		\caption{Performance of subspace clustering methods on EYaleB
			dataset. A 'NA' denotes that a running error was returned by
			the solver.  The data consists of face images under 64
			different illumination conditions of a randomly picked $n =
			\{2, 10, 20, 30, 38\}$ individuals. Images are downsampled
			from size $192 \times 168$ to size $48 \times 42$ and used
			as the feature vectors (data points).}
		\label{tbl:EYaleB-result}
		\begin{tabular}{c|ccccc}
			\hline
			No. subjects &   2   &  10   &  20   &  30   &             38              \\ \hline
			\multicolumn{6}{l}{\textsl{e\%: subspace-preserving representation error}} \\ \hline
			SSC-OMP    & 4.14  & 13.62 & 16.80 & 18.66 &            20.13            \\
			SSC-BP    & 2.70  & 10.33 & 12.67 & 13.74 &            14.64            \\
			LSR      & 22.77 & 67.07 & 79.52 & 84.94 &            87.57            \\
			LRSC     & 26.87 & 69.76 & 80.58 & 85.56 &            88.02            \\
			SCC      & 48.70 &  NA   &  NA   & 96.57 &            97.25            \\ \hline
			\multicolumn{6}{l}{\textsl{a\%: average clustering accuracy}}              \\ \hline
			SSC-OMP    & 99.18 & 86.09 & 81.55 & 78.27 &            77.59            \\
			SSC-BP    & 99.45 & 91.85 & 79.80 & 76.10 &            68.97            \\
			LSR      & 96.77 & 62.89 & 67.17 & 67.79 &            63.96            \\
			LRSC     & 94.32 & 66.98 & 66.34 & 67.49 &            66.78            \\
			SCC      & 78.91 &  NA   &  NA   & 14.15 &            12.80            \\ \hline
			\multicolumn{6}{l}{\textsl{t(sec.): running time}}               \\ \hline
			SSC-OMP    &  0.6  &  8.3  & 31.1  & 63.7  &       108.6       \\
			SSC-BP    & 49.1  & 228.2 & 554.6 & 1240  &       1851        \\
			LSR      &  0.1  &  0.8  &  3.1  &  8.3  &       15.9        \\
			LRSC     &  1.1  &  1.9  &  6.3  & 14.8  &       26.5        \\
			SCC      & 50.0  &  NA   &  NA   & 520.3 &       750.7       \\ \hline
		\end{tabular} 
	\end{table}
	
	The clustering performance of different methods is reported in Table \ref{tbl:EYaleB-result}. In terms of subspace-preserving recovery, we can observe a slightly better performance of SSC-BP over SSC-OMP in all cases. The other three methods have very large subspace-preserving representation errors especially when the number of subjects is $n \geq 10$. In terms of clustering accuracy, all methods do fairly well when the number of clusters is $2$ except for SCC, which is far worse than the others. As the number of subjects increases from $10$ to $38$, LSR and LRSC can only maintain an accuracy of about $60\%$ and SCC is even worse, but SSC-OMP and SSC-BP maintain a reasonably good performance, although the accuracy also degrades gradually. We can see that SSC-BP performs slightly better when the number of subjects is $2$ or $10$, but SSC-OMP performs better when $n > 10$.
	
	\vspace{-1mm}
	\section{Conclusion and Future Work}
	\vspace{-1mm}
	\label{sec:conclusion}
	We studied the sparse subspace clustering algorithm based on OMP. We derived theoretical conditions under which SSC-OMP is guaranteed to give a subspace-preserving representation. Our conditions are broader than those of state-of-the-art methods based on $\ell_2$ or nuclear norm regularization, and slightly weaker than those of SSC-BP. Experiments on synthetic and real world datasets showed that SSC-OMP is much more accurate than state-of-the-art methods based on $\ell_2$ or nuclear norm regularization and about twice as slow. On the other hand, SSC-OMP is slightly less accurate than SSC-BP but orders of magnitude faster. Moreover, we are one of the few \cite{Adler:SPL13,Adler:TNNLS15} that have demonstrated subspace clustering experiments on $100,\!000$ points. Overall, SSC-OMP provided the best accuracy versus computation trade-off for large scale subspace clustering problems. We note that while the optimization algorithm for SSC-BP in \cite{Elhamifar:TPAMI13} is inefficient for large scale problems, our most recent work \cite{You:CVPR16-EnSC} presents a scalable algorithm for elastic net based subspace clustering. A comparison with this work is left for future research.
	
	\myparagraph{Acknowledgments} 
	Work supported by NSF grant 1447822.\!\!\!
	
	\numberwithin{equation}{section}
	\numberwithin{figure}{section}
	\numberwithin{table}{section}
	
	\begin{appendices}
		
		In the appendices, we provide proofs for the theoretical results in the paper. We also provide the parameters of all the clustering methods studied in the handwritten digits and face image clustering experiments.
		
		\section{Proof of Theorem \ref{thm:SSC-OMP-independent}}
		
		In Theorem \ref{thm:SSC-OMP-independent} , we claim that the SSC-OMP gives subspace preserving representations if subspaces are independent. Here we provide the proof.
		
		\begin{theorem*}
			If the subspaces are independent, OMP gives a subspace-preserving representation of every data point.
		\end{theorem*}
		
		\begin{proof}
			Consider a data point $\x_j \in S_i$. We need to show that the output of $\OMP(X_{-j}, \x_j)$ is subspace-preserving. As an assumption, the termination parameters in OMP are set to be $\epsilon = 0$ and $k_{\text{max}} = N-1$ (\ie, the total number of points in the dictionary $X_{-j}$). This means, in particular, that OMP always terminates with some iteration $k^* \leq N-1$ with $q_{k^*} = 0$, which can be seen to hold as follows. If the OMP algorithm computes $q_k = 0$ for some $k \leq N-2$, then there is nothing to prove.  Thus, to complete the proof, we suppose that $q_k \neq 0$ for all $0 \leq k \leq N-2$, and proceed to prove that $q_{N-1} = 0$.  In the OMP algorithm, the columns of $X_{-j}$ indexed by $T_k$ for any $k$ are always linearly independent. This is evident from step \ref{step:residual} of Algorithm \ref{thm:SSC-OMP-independent}, as the residual vector $\q_k$ is orthogonal to every column of $X_{-j}$ indexed by $T_k$, thus when choosing a new entry to be added to $T_k$ in step \ref{step:best-residual} of Algorithm \ref{thm:SSC-OMP-independent}, points that are linearly dependent with the points indexed by $T_k$ would have zero inner product with $\q_k$, so would not be picked. Since all of the columns of $X_{-j}$ have been added by iteration $N-1$, we know that the columns of $X_{-j}$ are linearly independent and must contain at least $d_i$ linearly independent vectors from $S_i$~\footnote{We make the assumption that there are enough samples on each subspace. More specifically, $\forall i, \forall \x_j \in S_i, \rank(X_{-j}^i) = \dim(S_i)$.}. We conclude that $q_{k^*} = q_{N-1} = 0$ with $k^* = N-1$, as claimed. In light of this result and denoting $T^* := T_{k^*}$, it follows from $\q_{k^*}  = \0$ that $P_{T^*} \cdot \x_j = \x_j$ by line 4 of Algorithm 1, so that $\x_j$ is in the range of matrix $X_{T^*}$, which denotes the columns of $X_{-j}$ indexed by $T^*$. 
			
			
			As a consequence of the previous paragraph, the final output of OMP, given by
			\[
			\c^* = \argmin _{\c: \text{Supp}(\c) \subseteq T^*} \|\x_j - X_{-j} \c\|_2,
			\]
			will satisfy $\x_j = X_{-j} \cdot \c^*$. We rewrite it as
			\begin{equation}
			\x_j - \sum_{\substack{m: \x_m \in S_i\\ m \in T^*}} \x_m \cdot c^*_m = \sum_{\substack{m: \x_m \notin S_i\\ m \in T^*}} \x_m \cdot c^*_m.
			\label{eq:theorem1_eq1}
			\end{equation}
			
			Observe that the left hand side of \eqref{eq:theorem1_eq1} is in subspace $S_i$ while the right hand side is in subspace $S_{-i} := \sum_{m \ne i} S_m$. By the assumption that the set of all subspaces is independent, we know $S_i$ and $S_{-i}$ are also independent, so they intersect only at the origin. As a consequence, we have 
			\begin{equation}
			0 = \!\sum_{\substack{m: \x_m \notin S_i \\ m\in T^*}} \!\!\!\!\x_m \cdot c^*_m = \!\!\sum_{m: \x_m \notin S_i} \!\!\!\!\x_m \cdot c^*_m,
			\label{eq:theorem1_eq2}
			\end{equation}
			where we also used the fact that $c^*_m =0$ for all $m\notin T^*$.
			Combining \eqref{eq:theorem1_eq2} with the early fact that the columns of $X_{-j}$ indexed by $T_k$ are linearly independent for all $k$ (this includes $k = k^*$), we know that
			\begin{equation} \label{eq:c-zero}
			c_m ^* = 0 \ \text{ if $\x_m \notin S_i$ and $m\in T^*$.} 
			\end{equation}
			Finally, we use this to prove that $c^*$ is subspace-preserving.  To this end, suppose that $c^*_j \neq 0$, which from the definition of $c^*$ means that $j\in T^*$.  Using this fact, $c^*_j \neq 0$, and \eqref{eq:c-zero} allows us to conclude that $c^*_j \in S_i$.  Thus the solution $c^*$ is subspace-preserving.
		\end{proof}
		
		\section{Proof of Lemma \ref{lem:ss_individual}}
		
		In this section, we provide a detailed proof of Lemma \ref{lem:ss_individual}. The proof follows straight forwardly by comparing inductively the steps of the procedure  $\OMP(X_{-j}, \x_j)$ and the procedure of the fictitious problem $\OMP(X_{-j} ^i, \x_j)$. The idea is that these two procedures follow the same ``path'' if the condition of the lemma is satisfied.
		
		\begin{lemma}
			OMP gives a subspace-preserving representation for point $\x_j \in S_i$ in at most $d_i$ iterations if 
			\begin{equation}
			\label{eqn:lemma1_eq}
			\forall \w \in \cW_j^i
			\quad
			\max_{\x \in \bigcup_{k \ne i} \cX^k} | \w^\transpose \x | < \max_{\x \in \cX^i \backslash \{\x_j\}}|\w^\transpose \x | .
			\end{equation}
		\end{lemma}
		\begin{proof}
			Let $k^*$ be the number of iterations computed by the procedure $\OMP(X_{-j}, \x_j)$ so that $\q_{k^*} = \0$ (this was established in the first paragraph of the proof for Theorem \ref{thm:SSC-OMP-independent}). We prove that the solution to $\OMP(X_{-j}, \x_j)$ is subspace preserving by showing that $T_{k^*}$ only contains indexes of points from the $i$-th subspace. This is shown by induction, in the way that $T_{k}$ contains points from the $i$-th subspace for every $0 \le k \le k^*$.
			
			The set of residual directions $\cW_j ^i$ introduced in Definition \ref{def:residual-direction} plays an essential role in this proof. For notational clarity, we denote $\hat{\q}_k$ to be the residual vector generated at iteration $k$ of the algorithm $\OMP(X_{-j} ^i, \x_j)$ (note that this is the fictitious problem). The residual vectors of $\OMP(X_{-j}, \x_j)$ are denoted by $\q_k$. In the induction, we also show that $\OMP(X_{-j} ^i, \x_j)$ does not terminate at any $k < k^*$, and that $\q_k = \hat{\q}_k$ whenever $k \le k^*$.
			
			First, in the case of $k = 0$, the argument that $T_0$ only contains indexes of points that are from subspace $i$ is trivially satisfied since $T_0$ is empty. Also, $\q_0 = \hat{\q}_0$ is satisfied because they are both set to be $\x_j$ in line \ref{step:initialize} of Algorithm \ref{alg:OMP}.
			
			Now, given that $\q_k = \hat{\q}_k$ for some $k < k^*$ and that $T_k$ contains points only from subspace $S_i$, we show that $\q_{k+1} = \hat{\q}_{k+1}$ and that $T_{k+1}$ contains indexes of points from subspace $i$. This could be shown by noticing that the added entry in step \ref{step:best-residual} of Algorithm 1 is given by $\argmax\limits_{m \le N, m \ne j} |\x _m ^\transpose \q _k|$. Here, since $\q_k = \hat{\q}_k$, we have that $\q _k / \|\q_k\|_2$ is in the set $\cW_j ^i$. Then, by using condition \eqref{eqn:lemma1_eq}, we know that the $\argmax$ will give an index that corresponds to a point in $S_i$. This guarantees that $T_{k+1}$ only contains points from subspace $S_i$. Moreover, the picked point is evidently the same as the point picked at iteration $k$ of the $\OMP(X_{-j} ^i, \x_j)$. It then follows from step \ref{step:residual} of Algorithm \ref{alg:OMP} that the resultant residuals, $\q_{k+1}$ and $\hat{\q}_{k+1}$, are also equal. In the case of $k+1 < k^*$, this means that $\q_{k+1} = \hat{\q}_{k+1} \ne 0$, so the fictitious problem $\OMP(X_{-j} ^i, \x_j)$ does not terminate at this step. This finishes the mathematical induction.
			
			The fact that OMP terminates in at most $d_i$ iterations follows from the following facts: (i) we have established that $\OMP(X_{-j},\x_j)$ produces the same computations as does $\OMP(X_{-j}^i,\x_j)$; (ii)  the collection of vectors selected by $\OMP(X_{-j}^i,\x_j)$ are linearly independent and contained in subspace $S_i$; and (iii) the dimension of $S_i$ is equal to $d_i$.
		\end{proof}

		\section{Proof of Lemma \ref{lem:chain-inequalities}}
		
		In this section, we prove Lemma  \ref{lem:chain-inequalities} in Section \ref{sec:main-results}.
		
		\begin{lemma*}
			Let $\x_j \in S_i$. Then, for all $\w \in \cW_j^i$, we have:
			\begin{align}
			&\max_{\x \in \bigcup_{k \ne i} \!\cX^k} \!\! |\w^\transpose \x|
			\le \!\max_{k: k \ne i} \mu(\cW^i, \cX^k)
			\le \!\max_{k: k \ne i} \mu(\cX^i, \cX^k) / r_i;\nonumber\\
			&\max_{\x \in \cX^i \backslash \{\x_j\}} |\w ^\transpose \x|
			\ge r(\cP_{-j}^i) \ge r_i.
			\end{align}
		\end{lemma*}
		
		\begin{proof}
			Two of the inequalities need proofs while the other two follow directly from definitions. 
			
			For the first one, we prove that $\max_{k: k \ne i} \mu(\cW^i, \cX^k)$ $
			\le \max_{k: k \ne i} \mu(\cX^i, \cX^k) / r_i$. To do this, it suffices to show that for any $k \ne i$, $\mu(\cW^i, \cX^k) \le \mu(\cX^i, \cX^k) / r_i$ . Notice that any point $\hat{\w}$ in $\cW^i$ is in the subspace $S_i$, so it could be written as a linear combination of the points in $\cX^i$, \ie $\hat{\w} = X^i \cdot \c$ for some $\c$. Specifically, we pick a $\c$ that is given by the following optimization program: 
			\begin{equation}
			\hat{\c} = \argmin_{\c} \|\c\|_1 \st \hat{\w} = X^i \cdot \c.
			\label{eq:lemma2_eq1}
			\end{equation}
			Using \eqref{eq:lemma2_eq1}, defining $Y^{\transpose} = X^{k\transpose} \cdot X^i$, letting the $\ell$th column of $Y$ by denoted by $y_\ell$, and using the Cauchy-Schwarz inequality, we can observe that 
			\begin{align*}
			\|X^{k \transpose} \hat{\w}\|_\infty
			&= \|X^{k \transpose} \cdot X^i \cdot \hat{\c} \|_\infty
			=  \|Y^{\transpose} \cdot \hat{\c} \|_\infty \\
			&= \max_{\ell} |y_\ell^{\transpose} \hat{\c} | 
			\leq  \max_{\ell} \|y_\ell\|_{\infty} \|\hat{\c}\|_1 \\
			&= \|\hat{\c}\|_1 \max_{\ell} \|y_\ell\|_{\infty} \\
			&\leq \|\hat{\c}\|_1 \max_{\ell} \mu(\cX^i,\cX^k) \\ 
			&=  \mu(\cX^i, \cX^k) \cdot \|\hat{\c}\|_1.
			\end{align*}
			To proceed, we need to provide a bound on $\|\hat{\c}\|_1$. As $\hat{\c}$ is defined by \eqref{eq:lemma2_eq1}, it is shown that such a bound exists and is given by (see, e.g. lemma B.2 in \cite{Soltanolkotabi:AS14}) 
			\[
			\|\hat{\c}\|_1 \le \|\hat{\w}\|_2 / r(\cP^i) = 1 / r(\cP^i),
			\]
			where $\cP ^i := \conv(\pm \cX ^i)$, and we use the fact that every point in the set of residual directions $\cW^i$ is defined to have unit norm. Now, by the definitions we can get $r(\cP^i) \ge r(\cP_{-j}^i) \ge r_i$, thus $\|\hat{\c}\|_1 \le 1 / r(\cP^i) \le 1 / r_i$, which gives
			\begin{equation}
			\|X^{k \transpose} \hat{\w}\|_\infty \le  \mu(\cX^k, \cX^i) / r_i.
			\label{eq:lemma2_eq2}
			\end{equation}
			
			Finally, since \eqref{eq:lemma2_eq2} holds for any $\hat{\w}$ in $\cW^i$, the conclusion follows that $\mu(\cW^i, \cX^k) \le \mu(\cX^i, \cX^k) / r_i$.
			
			\bigskip
			
			For the second part, we prove that for all $\w \in \cW_j^i$, $\max_{\x \in \cX^i \backslash \{\x_j\}} |\w ^\transpose \x|
			\ge r(\cP_{-j}^i)$, or equivalently, $\| X_{-j} ^{i \transpose} \cdot \w\|_\infty \ge r(\cP_{-j}^i)$. The proof relies on the result (see definition 7.2 in \cite{Soltanolkotabi:AS12}) that for an arbitrary vector $\y \in S_i$,
			\[
			\|X_{-j} ^{i \transpose} \cdot \y\|_\infty \le 1 \Rightarrow \|\y\|_2 \le 1 / r(\cP_{-j}^i).
			\]
			It then follows that if (by contradiction) $\| X_{-j} ^{i \transpose} \cdot \w\|_\infty < r(\cP_{-j}^i)$, then $\| X_{-j} ^{i \transpose} \cdot \w\|_\infty = r(\cP_{-j}^i) - \epsilon > 0$ for some $\epsilon > 0$, and
			\[\begin{split}
			&\|X_{-j} ^{i \transpose} \frac{\w}{ r(\cP_{-j}^i) - \epsilon }\|_\infty = 1 \le 1\\
			\Rightarrow \
			&\| \frac{\w}{ r(\cP_{-j}^i) - \epsilon } \|_2 \le 1 / r(\cP_{-j}^i)  \\
			\Rightarrow \
			&\| \w \|_2 \le (r(\cP_{-j}^i) - \epsilon) / r(\cP_{-j}^i) < 1,
			\end{split}\]
			which contradicts the fact that $\w$ is normalized.
		\end{proof}

		\section{Proof of Theorem \ref{thm:deterministic_1} and Corollary \ref{thm:deterministic_2}}
		
		We explicitly show the proof of Theorem \ref{thm:deterministic_1} and Corollary \ref{thm:deterministic_2}. They follow from the previous two lemmas.
		
		\begin{theorem*}
			The output of OMP is subspace preserving if
			%
			\begin{equation}
			\forall i=1,\dots,n,\quad
			\max_{k: k \ne i} \mu(\cW^i,\cX ^k) < r_i.
			\end{equation} 
		\end{theorem*}
		
		\begin{corollary*}
			The output of OMP is subspace preserving if
			\begin{equation}
			\forall i=1,\dots,n,\quad
			\max_{k: k \ne i} \mu(\cX^i, \cX ^k) < r_i ^2.
			\end{equation} 
		\end{corollary*}
		
		\begin{proof}
			
			Notice from Lemma \ref{lem:ss_individual} that the solution of SSC-OMP for $\x_j \in S_i$ is subspace preserving if
			\begin{equation}
			\label{eq:Theorem2_eq1}
			\forall \w \in \cW_j^i
			\quad
			\max_{\x \in \bigcup_{k \ne i} \cX^k} | \w^\transpose \x | < \max_{\x \in \cX^i \backslash \{\x_j\}}|\w^\transpose \x | .
			\end{equation}
			
			Lemma \ref{lem:chain-inequalities} provides bounds for both sides of \eqref{eq:Theorem2_eq1} from which the theorem and the corollary follow.

		\end{proof}

		\section{Proof of Theorem \ref{thm:random1}}
		
		\begin{theorem*}
			Assume a random model where all subspaces are of equal dimension $d$ and the number of data points in each subspace is $\rho d + 1$, where $\rho > 1$ is the ``density", so the total number data points in all subspaces is $N = n  (\rho d + 1)$. The output of OMP is subspace preserving with probability $p > 1 - \frac{2d}{N} - N e^ {-\sqrt{\rho} d}$ if
			\begin{equation}
			d < \frac{c^2(\rho) \log \rho}{12} \frac{D}{\log N},
			\label{eq:theorem3_eq1}
			\end{equation}
			where $c(\rho) > 0$ is a constant that depends only on $\rho$. 
		\end{theorem*}
		
		%
		
		\begin{proof}
			The proof goes by providing bounds for the left and right hand side of the inequality in Theorem \ref{thm:deterministic_1}, copied here for convenience of reference:
			\begin{equation}
			\forall i=1,\dots,n,\quad
			\max_{k: k \ne i} \mu(\cW^i,\cX ^k) < r_i.
			\label{eq:theorem3_eq2}
			\end{equation}
			
			We first give a bound on the inradius $r_i$. Denote
			\[\begin{split}
			\bar{r} = \frac{c(\rho) \sqrt{\log \rho}}{\sqrt{2d}} 
			\ \ \text{and} \ \ \bar{p}_r = N e^ {-\sqrt{\rho} d},
			\end{split}\]
			in which $c(\rho)$ is a numerical constant depending on $\rho$. \cite{Soltanolkotabi:AS12} shows that since points in each subspace are independently distributed, it holds that
			\[
			P( r_i \ge \bar{r} \text{ for all } i) \ge 1 - \bar{p}_r.
			\]
			
			Next we give a bound on the coherence. From an upper bound on the area of a spherical cap \cite{Ball:97,Soltanolkotabi:AS12}, we have that if $\x, \y \in \Re ^D$ are two random vectors that are distributed uniformly and independently on the unit sphere, then
			\begin{equation}
			P \left\{|\langle \x, \y \rangle| \ge \sqrt{\frac{6\log N}{D}}\right\} \le \frac{2}{N^3}.
			\label{eq:theorem3_eq3}
			\end{equation}
			
			Under the random model, points $\x \in \cX^k, \forall k$ are distributed uniformly at random on the unit sphere of $\Re ^D$ by assumption. Any residual point $\w \in \cW^i, \forall i$ also has uniform distribution on the unit sphere as it depends only on points in $\cX^i$, which are independent and uniformly distributed. Furthermore, any pair of points $\x \in \cX^k$ and $\w \in \cW^i$ are distributed independently because points in $\cX^k$ and $\cX^i$ are independent. Thus the result of Equation \eqref{eq:theorem3_eq3} is applicable here. Since there are at most $d \times N^2$ pairs of inner product in $\mu(\cW ^i, \cX^k)$, by using the union bound we can get
			\[
			P\big( \mu(\cW ^i, \cX^k) \le \bar{\mu} \ \text{for all } i, k\big) \ge 1 - \bar{p}_{\mu}dN^2,
			\]
			where we have defined
			\[
			\bar{\mu} = \sqrt{\frac{6 \log N}{D}}
			\ \ \text{and} \ \ \bar{p}_{\mu} = 2/N^3.
			\]
			As a consequence, if the condition \eqref{eq:theorem3_eq1} holds then we have $\bar{r} < \bar{\mu}$. Applying again the union bound we get that condition \eqref{eq:theorem3_eq2} holds with probability $p > 1 - \bar{p}_{\mu}d - \bar{p}_r$. This finishes the proof. 
		\end{proof}

		\section{Comparison with prior work on SSC-OMP}
		
		In Theorem \ref{thm:deterministic_1} we give a sufficient condition for guaranteeing subspace-preserving of the SSC-OMP:
		\begin{equation}
		\forall i=1,\dots,n,\quad
		\max_{k: k \ne i} \mu(\cW^i,\cX ^k) < r_i,
		\label{eq:comparison_eq1}
		\end{equation} 
		and in Corollary \ref{thm:deterministic_2} a stronger sufficient condition:
		\begin{equation}
		\forall i=1,\dots,n,\quad
		\max_{k: k \ne i} \mu(\cX^i, \cX ^k) < r_i ^2.
		\label{eq:comparison_eq2}
		\end{equation} 
		
		Prior to this work, \cite{Dyer:JMLR13} gives another sufficient condition for SSC-OMP giving subspace-preserving representation, namely,
		\begin{equation}
		\label{eq:comparison_eq3}
		\max_{k:k \ne i} \mu(\cX ^i, \cX ^k) < r_i - \frac{2 \sqrt{1 - (r_i)^2}}{\sqrt[4]{12}} \max_{k: k \ne i} \cos \theta_{i, k}^*,
		\end{equation}
		in which the subspace angle is defined as
		\begin{equation}
		\theta_{i, k}^* = \min_{\substack{\x \in S_i\\ \|\x\|_2 = 1}} \min_{\substack{\y \in S_k\\ \|\y\|_2 = 1}} \arccos\langle \x, \y \rangle.
		\end{equation}
		
		We claim that Theorem \ref{thm:deterministic_1} in this work is a stronger result than that provided in the work \cite{Dyer:JMLR13}, as the sufficient condition of \eqref{eq:comparison_eq3} implies \eqref{eq:comparison_eq1}. Here we give a rigorous argument for this claim.
		
		Notice that the inequality in (\ref{eq:comparison_eq3}) implies that $\forall k \ne i$,
		\begin{equation}
		\mu(\cX ^i, \cX ^k) < r_i - \sqrt{2 - 2 r_i} \cos \theta_{i, k}^*,
		\label{eq:PriorOMP2}
		\end{equation}
		see Lemma 1 in their paper. We show that condition (\ref{eq:PriorOMP2}) implies (\ref{eq:comparison_eq1}) when $r_i \le 1/2$, and implies condition (\ref{eq:comparison_eq2}) when $r_i > 1/2$, which means that their result is weaker than our result that is based on condition (\ref{eq:comparison_eq1}).
		
		Case 1. If $r_i \le 1/2$, then $\sqrt{2 - 2 r_i}\ge 1$, thus
		\[\begin{split}
		(\ref{eq:PriorOMP2}) &\Rightarrow 
		\mu(\cX ^i, \cX ^k) < r_i - \cos \theta_{i, k}^*
		\Rightarrow 
		\cos \theta_{l, k}^* < r_i\\
		&\Rightarrow
		\mu(\cX ^k, \cW^i) < r_i \Leftrightarrow (\ref{eq:comparison_eq1}).
		\end{split}\]
		
		Case 2. If $r_i > 1/2$, then
		\[\begin{split}
		(\ref{eq:PriorOMP2}) &\Rightarrow 
		\mu(\cX ^i, \cX ^k) < r_i - \sqrt{2 - 2 r_i} \mu(\cX ^i, \cX ^k)\\
		&\Rightarrow 
		\mu(\cX ^i, \cX ^k) < r_i / (1 + \sqrt{2 - 2 r_i})\\
		&\Rightarrow
		\mu(\cX ^i, \cX ^k) < r_i / (1 + (2 - 2 r_i))\\
		&\Rightarrow
		\mu(\cX ^i, \cX ^k) < (r_i)^2 \Leftrightarrow 
		(\ref{eq:comparison_eq2}) \Rightarrow (\ref{eq:comparison_eq1}).
		\end{split}\]
		So the condition in (\ref{eq:comparison_eq1}) is implied by (\ref{eq:comparison_eq3}). 
		
		\section{Parameters for real experiments}
		
		For the purpose of reproducible results, we report the parameters used for all the methods in the real data experiments. For OMP, we set $\epsilon$ in Algorithm \ref{alg:OMP} to be $10^{-3}$, $k_{\max}$ to be the true subspace dimension in the synthetic experiments, $10$ in digit clustering and $5$ in face clustering. For LSR, we use ``LSR2'' in \cite{Lu:ECCV12} with regularization $\lambda=60$ for digit clustering and $\lambda=0.3$ for face clustering. For LRSC, we use model ``P$_3$'' in \cite{Vidal:PRL14}, with parameters $\tau = \alpha = 0.1$ for digit clustering and $\tau = \alpha = 150$ for face clustering. For SCC, we use dimension $d=8$ for digit clustering and $5$ for face clustering. We use $\ell_1$-Magic for SSC-BP in the synthetic experiments. For digit and face clustering, we use the noisy variation of SSC-BP in \cite[sec. 3.1]{Elhamifar:TPAMI13} for digit clustering with $\lambda_z = 80 / \mu_z$, and the sparse outlying entries variation of SSC-BP in \cite[sec. 3.1]{Elhamifar:TPAMI13} for face clustering with $\lambda_e =30 / \mu_e$. For all algorithms, these constants were chosen to optimize performance
		
		For a fair comparison, we allow standard pre/post-processing to be used whenever they improve the clustering accuracy. For preprocessing, we allow normalization of the original data points using the $\ell_2$ norm, and for post-processing, we allow normalization of the coefficient vectors using the $\ell_\infty$ norm.  For experiments on synthetic data, we do not use any pre/post-processing. In digit clustering, preprocessing is applied to SSC-BP and SCC, and post-processing is used for SSC-OMP and SSC-BP. For face clustering, preprocessing is applied to SSC-OMP, LSR, LRSC and SCC, while post-processing is used for SSC-BP and LRSC.

	\end{appendices}
	
	{\small
		\bibliographystyle{ieee}
		\bibliography{biblio/vidal,biblio/vision,biblio/math,biblio/learning,biblio/sparse,biblio/geometry,biblio/dti,biblio/recognition,biblio/surgery,biblio/coding,biblio/matrixcompletion,biblio/segmentation}
	}
	
\end{document}